\documentclass{article}
\usepackage{algorithm}
\usepackage{algorithmicx}
\usepackage[noend]{algpseudocode}
\usepackage{amsmath}
\usepackage{amssymb}
\usepackage{amsthm}
\usepackage{bbm}
\usepackage{bm}
\usepackage{color}
\usepackage{dsfont}
\usepackage{enumerate}
\usepackage{graphicx}
\usepackage[final,nonatbib]{nips_2017}
\usepackage{subfigure}
\usepackage{times}

\usepackage[usenames,dvipsnames]{xcolor}
\usepackage[bookmarks=false]{hyperref}
\hypersetup{
  pdftex,
  pdffitwindow=true,
  pdfstartview={FitH},
  pdfnewwindow=true,
  colorlinks,
  linktocpage=true,
  linkcolor=Green,
  urlcolor=Green,
  citecolor=Green
}

\usepackage[backgroundcolor=White]{todonotes}

\usepackage[capitalize,noabbrev]{cleveref}

\newcommand{\commentout}[1]{}
\newcommand{\junk}[1]{}
\newcommand{\etal}{\emph{et al.}}

\newtheorem{theorem}{Theorem}

\newtheorem{lemma}{Lemma}

\newcommand{\cE}{\mathcal{E}}
\newcommand{\ccE}{\overline{\cE}}

\newcommand{\eps}{\varepsilon}

\newcommand{\realset}{\mathbb{R}}

\newcommand{\colvar}{\textsc{v}}

\newcommand{\rowvar}{\textsc{u}}

\newcommand{\abs}[1]{\left|#1\right|}
\newcommand{\ceils}[1]{\left\lceil#1\right\rceil}
\newcommand{\condE}[2]{\mathbb{E} \left[#1 \,\middle|\, #2\right]}

\newcommand{\E}[1]{\mathbb{E} \left[#1\right]}

\newcommand{\I}[1]{\mathds{1} \! \left\{#1\right\}}

\newcommand{\maxnorm}[1]{\left\|#1\right\|_\infty}

\newcommand{\normw}[2]{\left\|#1\right\|_{#2}}

\newcommand{\rnd}[1]{\mathbf{#1}}
\newcommand{\set}[1]{\left\{#1\right\}}

\newcommand{\transpose}{^\mathsf{\scriptscriptstyle T}}

\DeclareMathOperator*{\argmax}{arg\,max\,}

\let\det\relax
\DeclareMathOperator{\det}{det}
\DeclareMathOperator{\poly}{poly}
\mathchardef\mhyphen="2D

\newcommand{\lowrankelim}{{\tt LowRankElim}}

\newcommand{\ucb}{{\tt UCB1}}

\begin{document}

\title{Stochastic Low-Rank Bandits}

\author{Branislav Kveton \\
Adobe Research \\
San Jose, CA \\
\emph{kveton@adobe.com} \And
Csaba Szepesv\'ari \\
Department of Computing Science \\
University of Alberta \\
\emph{szepesva@cs.ualberta.ca} \And
Anup Rao \\
Adobe Research \\
San Jose, CA \\
\emph{anuprao@adobe.com} \And
Zheng Wen \\
Adobe Research \\
San Jose, CA \\
\emph{zwen@adobe.com} \And
Yasin Abbasi-Yadkori \\
Adobe Research \\
San Jose, CA \\
\emph{abbasiya@adobe.com} \And
S. Muthukrishnan \\
Department of Computer Science \\
Rutgers University \\
\emph{muthu@cs.rutgers.edu}}

\maketitle

\begin{abstract}
Many problems in computer vision and recommender systems involve low-rank matrices. In this work, we study the problem of finding the maximum entry of a stochastic low-rank matrix from sequential observations. At each step, a learning agent chooses pairs of row and column arms, and receives the noisy product of their latent values as a reward. The main challenge is that the latent values are unobserved. We identify a class of non-negative matrices whose maximum entry can be found statistically efficiently and propose an algorithm for finding them, which we call $\lowrankelim$. We derive a $O((K + L) \poly(d) \Delta^{-1} \log n)$ upper bound on its $n$-step regret, where $K$ is the number of rows, $L$ is the number of columns, $d$ is the rank of the matrix, and $\Delta$ is the minimum gap. The bound depends on other problem-specific constants that clearly do not depend $K L$. To the best of our knowledge, this is the first such result in the literature.
\end{abstract}


\section{Introduction}
\label{sec:introduction}

We study the problem of finding the maximum entry of a stochastic low-rank matrix from sequential observations. Many real-world problems, especially in recommender systems \cite{koren09matrix,ricci11introduction}, are known to have an approximately low-rank structure. Therefore, we believe that our problem has ample applications. For instance, consider a marketer who wants to design a campaign that maximizes the click-through rate (CTR). The actions of the marketer are pairs of products and user segments. Let the product and user segment be the row and column of a matrix, where each entry is the CTR of a given segment on a given product. Then the maximum entry of this matrix is the solution to our problem. This matrix is expected to be low rank because similar segments tend to react similarly to similar products.

We propose an online learning model for our motivating problem, which we call a \emph{stochastic low-rank bandit}. The learning agent interacts with our problem as follows. At time $t$, the agent chooses pairs of row and column arms, and receives the noisy product of their latent values as a reward. The main challenge of our problem is that the latent values are not revealed. The goal of the agent is to maximize its expected cumulative reward, or equivalently to minimize its expected cumulative regret with respect to the most rewarding solution in hindsight.

We make three major contributions. First, we formulate the online learning problem of \emph{stochastic low-rank bandits}, on a class of non-negative rank-$d$ matrices that can be solved statistically efficiently. Second, we design an elimination algorithm, $\lowrankelim$, for solving it. The key idea in $\lowrankelim$ is to explore all remaining row and column $d$-subsets randomly over all remaining column and row $d$-subsets, respectively, to estimate their expected rewards; and then eliminate suboptimal $d$-subsets. Our algorithm is computationally and sample efficient when the rank is small, such as $d \leq 4$. Third, we derive a $O((K + L) \poly(d) \Delta^{-1} \log n)$ gap-dependent upper bound on the $n$-step regret of $\lowrankelim$, where $K$ is the number of rows, $L$ is the number of columns, $d$ is the rank of the matrix, and $\Delta$ is the minimum of the row and column gaps. This result is stated in \cref{thm:upper bound} in \cref{sec:analysis}. The bound also depends on problem-specific constants that clearly do not depend on $K L$. One of our main contributions is that we identify the right notion of the gap.

We denote random variables by boldface letters and define $[n] = \set{1, \dots, n}$. For any two sets $A$ and $B$, we denote by $A^B$ the set of all vectors whose entries are indexed by $B$ and take values from $A$. Let $\Pi_k(A)$ be the \emph{set of all $k$-subsets} of set $A$. Let $M \in [0, 1]^{K \times L}$ be any matrix. Then we denote by $M(I, :)$ its submatrix of $k$ rows $I \in [K]^k$, by $M(:, J)$ its submatrix of $\ell$ columns $J \in [L]^\ell$; and by $M(I, J)$ its submatrix of rows $I$ and columns $J$. When $I$ and $J$ are sets, we assume that the rows and columns of $M$ are ordered in any fixed order, such ascending. We denote by $\mathcal{S}_d$ the set of points in the standard $d$-dimensional simplex, $\mathcal{S}_d = \set{v \in [0, 1]^d: \normw{v}{1} \leq 1}$; and by $\mathcal{S}_{n, d}$ the set of $n \times d$ matrices whose rows are from $\mathcal{S}_d$, $\mathcal{S}_{n, d} = \set{M \in [0, 1]^{n \times d}: M(i, :) \in \mathcal{S}_d \text{ for all } i \in [n]}$.


\section{Setting}
\label{sec:setting}

We formulate our learning problem as a \emph{stochastic low-rank bandit}. An instance of this problem is defined by a tuple $(U, V, P)$, where $U \in \mathcal{S}_{K, d}$ are \emph{latent row factors}, $V \in \mathcal{S}_{L, d}$ are \emph{latent column factors}, $K$ is the number of rows, $L$ is the number of columns, $d \ll \min \set{K, L}$ is the rank of $\bar{R} = U V\transpose \in [0, 1]^{K \times L}$, and $P$ is a distribution over the entries of $\bar{R}$. We assume that the stochastic reward of arm $(i, j)$ at time $t$, $\rnd{r}_t(i, j) \in [0, 1]$, satisfies $\E{\rnd{r}_t(i, j)} = \bar{R}(i, j)$. Let
\begin{align}
  \textstyle
  (i^\ast, j^\ast) = \argmax_{(i, j) \in [K] \times [L]} \bar{R}(i, j)
  \label{eq:maximum entry}
\end{align}
be the maximum entry of $\bar{R}$. The problem of learning $(i^\ast, j^\ast)$ from noisy observations of $\bar{R}$ is challenging, in the sense that no statistically-efficient learning algorithm exists for solving all instances of this problem (\cref{sec:related work}). In this work, we make two assumptions that allow us to make progress towards statistical efficiency.

\subsection{Hott Topics}
\label{sec:hott topics}

Our first key assumption is that $\bar{R}$ is a hott topics matrix \cite{recht12factoring}. Specifically, we assume that there exist $d$ \emph{base row factors}, $U(I^\ast, :)$ for some $I^\ast \in \Pi_d([K])$, such that all rows of $U$ can be written as a convex combination of the rows of $U(I^\ast, :)$ and the zero vector; and that there exist $d$ \emph{base column factors}, $V(J^\ast, :)$ for some $J^\ast \in \Pi_d([L])$, such that all rows of $V$ can be written as a convex combination of the rows of $V(J^\ast, :)$ and the zero vector. Without loss of generality, we assume that $I^\ast = J^\ast = [d]$; and denote the corresponding row and column factors by $U^\ast = U(I^\ast, :)$ and $V^\ast = V(J^\ast, :)$, respectively.

\begin{figure*}[t]
  \includegraphics[viewport=0.5in 0.5in 8in 2.5in,width=5in]{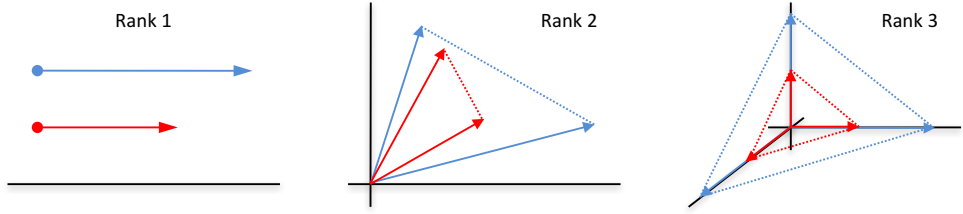}
  \caption{Visualization of optimal (blue) and suboptimal (red) $d$-rows in up to three dimensions. The vectors are individual rows in a $d$-row. The volume corresponding to the $d$-row is marked with dotted lines.}
  \label{fig:latent space}
\end{figure*}

Based on our assumption, $(i^\ast, j^\ast) \in I^\ast \times J^\ast$. The claim that $i^\ast \in I^\ast$ follows from the observation that for any column $j$,
\begin{align*}
  \max_{i \in [K]} U(i, :) V(j, :)\transpose \leq
  \max_{z \in \mathcal{S}_d} z U^\ast V(j, :)\transpose =
  \max_{k \in [d]} U^\ast(k, :) V(j, :)\transpose\,.
\end{align*}
The claim that $j^\ast \in J^\ast$ is proved analogously.

\subsection{Simplified Problem}
\label{sec:simplified problem}

Our second key assumption is that we study a related problem to \eqref{eq:maximum entry}, learning of $(I^\ast, J^\ast)$. When $(I^\ast, J^\ast)$ is known, learning of $(i^\ast, j^\ast)$ is a problem with $d^2$ arms, which is small in comparison to our original problem with $K L$ arms. The learning agent interacts with our new problem as follows. At time $t$, the agent chooses \emph{arm} $(\rnd{I}_t, \rnd{J}_t) \in \Pi_d([K]) \times \Pi_d([L])$, a pair of $d$-subsets of rows and columns, and \emph{observes} a noisy realization of matrix $\bar{R}(\rnd{I}_t, \rnd{J}_t)$, $\rnd{r}_t(i, j)$ for all $(i, j) \in \rnd{I}_t \times \rnd{J}_t$. The \emph{reward} is $\rnd{r}_t(i^\ast(\rnd{I}_t, \rnd{J}_t), j^\ast(\rnd{I}_t, \rnd{J}_t))$, where
\begin{align*}
  (i^\ast(I, J), j^\ast(I, J)) = \argmax_{(i, j) \in I \times J} \bar{R}(i, j)
\end{align*}
for any $(I, J) \in \Pi_d([K]) \times \Pi_d([L])$. To simplify language, we refer to the $d$-subsets of rows and columns as a \emph{$d$-row} and \emph{$d$-column}, respectively.

The objective of the learning agent is to minimize its \emph{expected cumulative regret} in $n$ steps $\mathcal{R}(n) = \E{\sum_{t = 1}^n \mathcal{R}(\rnd{I}_t, \rnd{J}_t)}$, where $\mathcal{R}(\rnd{I}_t, \rnd{J}_t) = \rnd{r}_t(i^\ast, j^\ast) - \rnd{r}_t(i^\ast(\rnd{I}_t, \rnd{J}_t), j^\ast(\rnd{I}_t, \rnd{J}_t))$ is the \emph{instantaneous stochastic regret} of the agent at time $t$.\footnote{Our regret bound in \cref{thm:upper bound} also holds for $\mathcal{R}(\rnd{I}_t, \rnd{J}_t) = \sum_{i, j = 1}^d \rnd{r}_t(i, j) - \sum_{(i, j) \in \rnd{I}_t \times \rnd{J}_t} \rnd{r}_t(i, j)$. This is another natural definition of the regret. The proof changes only in the first inequality in \cref{sec:regret decomposition lemma}.}


\section{Noise-Free Problem}
\label{sec:noise-free problem}

This section shows that the problem of finding the maximum entry of a noise-free low-rank matrix can be viewed as an elimination problem. We focus on row elimination. The column elimination is analogous. We start with rank-$1$ matrices. The maximum entry of a non-negative rank-$1$ matrix is in the row with the highest latent value \cite{katariya17stochastic}. Therefore, row $i$ can be eliminated by row $i'$ when $U(i, 1) < U(i', 1)$, when the length of $U(i, :)$ is lower than the length of $U(i', :)$ (\cref{fig:latent space}a).

A natural generalization of the length in a one-dimensional space is the area in a two-dimensional space. Therefore, in our class of rank-$2$ matrices, a pair of rows $I \in \Pi_2([K])$ can be eliminated by a pair of rows $I' \in \Pi_2([K])$ when the simplex over the rows of $U(I, :)$ has a smaller area than that over the rows of $U(I', :)$, as shown in \cref{fig:latent space}b. This follows from our assumption that any $U(I, :)$ can be written as $U(I, :) = Z U^\ast$ for some $Z \in \mathcal{S}_{d, d}$.

Generally, in any rank-$d$ matrix in our class of matrices, a $d$-row $I$ can be eliminated by a $d$-row $I'$ when the simplex over the rows of $U(I, :)$ has a smaller volume than that over the rows of $U(I', :)$, as shown in \cref{fig:latent space}c. The volume of the simplex over the rows of $U(I, :)$ is $d!^{-1} \abs{\det(U(I, :))}$. In the rest of this work, we neglect the factor of $d!$. This has no impact on elimination because this factor is common among all simplex volumes.

Unfortunately, the above approach cannot be implemented because $U$ is not observed, as we only observe the entries of $\bar R$. Therefore, we estimate $\abs{\det(U(I, :))}$ from $\bar{R}(I, J)$, where $\bar{R}(I, J)$ are the observations of $d$-row $I \in \Pi_d([K])$ over $d$-column $J \in \Pi_d([L])$. In particular, from the definition of $\bar{R}$ and the properties of the determinant,
\begin{align*}
  \det(\bar{R}(I, J)) =
  \det(U(I, :) V(J, :)\transpose) =
  \det(U(I, :)) \det(V(J, :))
\end{align*}
for any $I \in \Pi_d([K])$ and $J \in \Pi_d([L])$. This implies that $\det(\bar{R}(I, J))$ can be viewed as a scaled observation of $\det(U(I, :))$; and that
\begin{align*}
  \det^2(\bar{R}(I, J)) < \det^2(\bar{R}(I', J)) \implies
  \det^2(U(I, :)) < \det^2(U(I', :))
\end{align*}
for any $d$-rows $I$ and $I'$, as long as $\det^2(V(J, :)) > 0$.

\begin{algorithm}[t]
  \caption{Finding the maximum entry of a noise-free rank-$d$ matrix.}
  \label{alg:noise-free max}
  \begin{algorithmic}[1]
    \State Choose any $d$-column, $J_1 \in \Pi_d([L])$, and observe it in all rows
    \State Choose any $d$-row, $I_1 \in \Pi_d([K])$, and observe it in all columns
    \Statex
    \State $I^\ast \gets [d]$ \Comment{Row elimination}
    \ForAll{$I \in \Pi_d([K])$}
      \If{$\det^2(\bar{R}(I, J_1)) > \det^2(\bar{R}(I^\ast, J_1))$}
        $I^\ast \gets I$
      \EndIf
    \EndFor
    \State $J^\ast \gets [d]$ \Comment{Column elimination}
    \ForAll{$J \in \Pi_d([L])$}
      \If{$\det^2(\bar{R}(I_1, J)) > \det^2(\bar{R}(I_1, J^\ast))$}
        $J^\ast \gets J$
      \EndIf
    \EndFor
  \end{algorithmic}
\end{algorithm}

The above reasoning leads to a particularly simple algorithm for solving the noise-free variant of our problem, which is presented in \cref{alg:noise-free max}. The algorithm is guaranteed to identify $(I^\ast, J^\ast)$ under the assumption that the \emph{minimum volume}
\begin{align}
  c_{\min} = \min \set{\min_{I \in \Pi_d([K])} \det^2(U(I, :)), \ \min_{J \in \Pi_d([L])} \det^2(V(J, :))}
  \label{eq:minimum volume}
\end{align}
is positive. This means that any $d$ rows and columns of $\bar{R}$ are linearly independent. We discuss how to alleviate the dependence on $c_{\min}$ in \cref{sec:discussion}.



\section{Noisy Problem}
\label{sec:noisy problem}

\cref{alg:noise-free max} is expected to perform poorly in the noisy setting. The challenge is that a single noisy realization of $\bar{R}(:, J_1)$ and $\bar{R}(I_1, :)$ is unlikely to be sufficient to learn $I^\ast$ and $J^\ast$. This issue can be addressed by observing multiple noisy realizations of $\bar{R}(:, J_1)$ and $\bar{R}(I_1, :)$, and then acting on their empirical averages. This approach is problematic for two reasons. First and foremost, when $I_1$ and $J_1$ are chosen poorly, $\det^2(U(I_1, :))$ and $\det^2(V(J_1, :))$ are close to zero, and many observations are needed to learn $I^\ast$ and $J^\ast$. Second, it is wasteful in the sense that some $d$-rows and $d$-columns can be detected as suboptimal from much less observations than the others. We propose an adaptive elimination algorithm that addresses these challenges in the next section.

\subsection{Algorithm $\lowrankelim$}
\label{sec:algorithm}

We propose an elimination algorithm \cite{auer10ucb} for finding the maximum entry of a noisy low-rank matrix, which maintains $\ucb$ confidence intervals \cite{auer02finitetime} on the scaled volumes of all $d$-rows and $d$-columns. The algorithm is presented in \cref{alg:noisy max} and we call it $\lowrankelim$. The algorithm operates in stages, which quadruple in length. In each stage, $\lowrankelim$ explores all remaining rows and columns randomly over all remaining $d$-columns and $d$-rows, respectively. At the end of the stage, it eliminates all $d$-rows and $d$-columns that cannot be optimal with a high probability. We denote the remaining $d$-rows and $d$-columns in stage $\ell$ by $\rnd{A}^\rowvar_\ell$ and $\rnd{A}^\colvar_\ell$, respectively. The row and column variables are distinguished by their upper indices, which are $\rowvar$ and $\colvar$, respectively.

\begin{algorithm}[t]
  \caption{$\lowrankelim$ for finding the maximum entry of a noisy rank-$d$ matrix.}
  \label{alg:noisy max}
  \begin{algorithmic}[1]
    \State $\tilde{\Delta}_0 \gets 1$, \ $\rnd{A}^\rowvar_0 \gets \Pi_d([K])$, \ $\rnd{A}^\colvar_0 \gets \Pi_d([L])$
    \Comment{Initialization}
    \Statex 
    \For{$\ell = 0, 1, \dots$}
      \State $n_\ell \gets \ceils{4 \tilde{\Delta}_\ell^{-2} C(n)}$
      \Statex
      \For{$t = 1, \dots, n_\ell$}
        \State Choose random $d$-row $\rnd{I}_t \in \rnd{A}^\rowvar_\ell$ and
        $d$-column $\rnd{J}_t \in \rnd{A}^\colvar_\ell$
        \For{$k = 1, 2$}
          \ForAll{$i \in \bigcup_{I \in \rnd{A}^\rowvar_\ell} I$}
          \Comment{Row exploration}
            \State Choose any $d$-row $I \in \rnd{A}^\rowvar_\ell$ such that $i \in I$
            \State Observe $d$-row $I$ over $\rnd{J}_t$ and store it in $\rnd{R}^\rowvar_{\ell, t, k}(I, :)$
          \EndFor
          \ForAll{$j \in \bigcup_{J \in \rnd{A}^\colvar_\ell} J$}
          \Comment{Column exploration}
            \State Choose any $d$-column $J \in \rnd{A}^\colvar_\ell$ such that $j \in J$
            \State Observe $d$-column $J$ over $\rnd{I}_t$ and store it in $\rnd{R}^\colvar_{\ell, t, k}(J, :)$
          \EndFor
        \EndFor
      \EndFor
      \Statex
      \State $\delta_\ell \gets \sqrt{C(n) n_\ell^{-1}}$
      \ForAll{$I \in \rnd{A}^\rowvar_\ell$}
      \Comment{UCBs and LCBs of all remaining $d$-rows}
        \State $\hat{\rnd{\mu}}^\rowvar_\ell(I) \gets n_\ell^{-1} \sum_{t = 1}^{n_\ell}
        \det(\rnd{R}^\rowvar_{\ell, t, 1}(I, :)) \det(\rnd{R}^\rowvar_{\ell, t, 2}(I, :))$
        \State $\rnd{U}^\rowvar_\ell(I) \gets \hat{\rnd{\mu}}^\rowvar_\ell(I) + \delta_\ell$, \
        $\rnd{L}^\rowvar_\ell(I) \gets \hat{\rnd{\mu}}^\rowvar_\ell(I) - \delta_\ell$
      \EndFor
      \ForAll{$J \in \rnd{A}^\colvar_\ell$}
      \Comment{UCBs and LCBs of all remaining $d$-columns}
        \State $\hat{\rnd{\mu}}^\colvar_\ell(J) \gets n_\ell^{-1} \sum_{t = 1}^{n_\ell}
        \det(\rnd{R}^\colvar_{\ell, t, 1}(J, :)) \det(\rnd{R}^\colvar_{\ell, t, 2}(J, :))$
        \State $\rnd{U}^\colvar_\ell(J) \gets \hat{\rnd{\mu}}^\colvar_\ell(J) + \delta_\ell$, \
        $\rnd{L}^\colvar_\ell(J) \gets \hat{\rnd{\mu}}^\colvar_\ell(J) - \delta_\ell$
      \EndFor
      \Statex
      \State $\rnd{A}^\rowvar_{\ell + 1} \gets \rnd{A}^\rowvar_\ell$, \
      $\rnd{I}^\ast_\ell \gets \argmax_{I \in \rnd{A}^\rowvar_\ell} \rnd{L}^\rowvar_\ell(I)$
      \Comment{$d$-row elimination}
      \ForAll{$I \in \rnd{A}^\rowvar_\ell$}
        \If{$\rnd{U}^\rowvar_\ell(I) \leq \rnd{L}^\rowvar_\ell(\rnd{I}^\ast_\ell)$}
          $\rnd{A}^\rowvar_{\ell + 1} \gets \rnd{A}^\rowvar_{\ell + 1} \setminus \set{I}$
        \EndIf
      \EndFor
      \State $\rnd{A}^\colvar_{\ell + 1} \gets \rnd{A}^\colvar_\ell$, \
      $\rnd{J}^\ast_\ell \gets \argmax_{J \in \rnd{A}^\colvar_\ell} \rnd{L}^\colvar_\ell(J)$
      \Comment{$d$-column elimination}
      \ForAll{$J \in \rnd{A}^\colvar_\ell$}
        \If{$\rnd{U}^\colvar_\ell(J) \leq \rnd{L}^\colvar_\ell(\rnd{J}^\ast_\ell)$}
          $\rnd{A}^\colvar_{\ell + 1} \gets \rnd{A}^\colvar_{\ell + 1} \setminus \set{J}$
        \EndIf
      \EndFor
      \Statex
      \State $\tilde{\Delta}_{\ell + 1} \gets \tilde{\Delta}_\ell / 2$
    \EndFor
  \end{algorithmic}
\end{algorithm}

Each stage of \cref{alg:noisy max} has three main steps: exploration, estimation, and elimination. In the exploration step (lines $4$--$12$), all remaining rows and columns are explored over $n_\ell$ random remaining $d$-columns and $d$-rows, respectively. The row and column observations are stored in matrices $\rnd{R}^\rowvar_{\ell, t, k} \in \realset^{K \times d}$ and $\rnd{R}^\colvar_{\ell, t, k} \in \realset^{L \times d}$, respectively. Therefore, the maximum number of observations in stage $\ell$ is $2 (K + L) d^2 n_\ell$. The separation of row and column observations is necessary to guarantee that the row and column estimators are scaled by the same factor, as in \cref{sec:noise-free problem}.

In the estimation step (lines $13$--$19$), $\lowrankelim$ estimates high-probability upper and lower confidence bounds on the scaled volumes of all remaining $d$-rows and $d$-columns. The scaled volume of $d$-row $I \in \rnd{A}^\rowvar_\ell$ is estimated as $\hat{\rnd{\mu}}^\rowvar_\ell(I)$ in line $13$. Since $\rnd{R}^\rowvar_{\ell, t, 1}(I, :)$ and $\rnd{R}^\rowvar_{\ell, t, 2}(I, :)$ are independent noisy observations of $\bar{R}(I, \rnd{J}_t)$, it is easy to show that
\begin{align*}
  \E{\hat{\rnd{\mu}}^\rowvar_\ell(I)} =
  \abs{\rnd{A}^\colvar_\ell}^{-1} \sum_{J \in \rnd{A}^\colvar_\ell} \det^2(\bar{R}(I, J))
\end{align*}
for any remaining $d$-columns $\rnd{A}^\colvar_\ell$. Also note that any realization of $\det(\rnd{R}^\rowvar_{\ell, t, k}(I, :))$ is reasonably bounded for small $d$. In particular, let $\det_{\max}(d) = \max_{M \in [0, 1]^{d \times d}} \det(M)$ be the maximum determinant of a $d \times d$ matrix on $[0, 1]$. Then $\abs{\det(\rnd{R}^\rowvar_{\ell, t, k}(I, :))} \leq \det_{\max}(d)$, where $\det_{\max}(d)$ is $1$, $1$, $2$, and $3$ when $d$ is $1$, $2$, $3$, and $4$, respectively. Therefore, when $d$ is small, we can argue that $\hat{\rnd{\mu}}^\rowvar_\ell(I)$ concentrates at $\bar{\rnd{\mu}}^\rowvar_\ell(I)$ by standard concentration inequalities for bounded i.i.d. random variables.

In the elimination step (lines $20$--$25$), $\lowrankelim$ eliminates suboptimal $d$-rows and $d$-columns. The confidence intervals are designed such that $\rnd{U}^\rowvar_\ell(I) \leq \rnd{L}^\rowvar_\ell(\rnd{I}^\ast_\ell)$ implies that $d$-row $I$ is suboptimal with a high probability for any column elimination policy up to the end of stage $\ell$, and $\rnd{U}^\colvar_\ell(J) \leq \rnd{L}^\colvar_\ell(\rnd{J}^\ast_\ell)$ implies that $d$-column $J$ is suboptimal with a high probability for any row elimination policy up to the end of stage $\ell$. As a result, all eliminations are correct with a high probability.

The computational complexity of the estimation and elimination steps (lines $13$--$25$) is exponential in $d$. Therefore, they can be implemented efficiently only for small $d$. The confidence radii depend on $\log n$ through
\begin{align}
  C(n) = 4 \det_{\max}^2(d) \log((K^d + L^d) n)\,.
  \label{eq:log n}
\end{align}
Since $K^d + L^d \leq (K + L)^d$, $C(n)$ is at most linear in $d$.


\section{Analysis}
\label{sec:analysis}

This section has three parts. In \cref{sec:upper bound}, we present a gap-dependent upper bound on the $n$-step regret of $\lowrankelim$. In \cref{sec:key lemmas}, we state our key lemmas and sketch their proofs. In \cref{sec:discussion}, we discuss the results of our analysis.

\subsection{Upper Bound}
\label{sec:upper bound}

$\lowrankelim$ seems to be a reasonable generalization of \cref{alg:noise-free max} to the noisy setting, where the scaled estimates of volumes are substituted with their upper and lower confidence bounds. As a result, it is expected that $\lowrankelim$ eliminates all suboptimal $d$-rows and $d$-columns as the number of stages increases, as long as all confidence bounds hold with a high probability. In this section, we derive a finite-time upper bound on the regret of $\lowrankelim$.

We measure the regret of $\lowrankelim$ by several metrics. Let $I \in \Pi_d([K])$ be suboptimal $d$-row and $J \in \Pi_d([L])$ be suboptimal $d$-column. Then the \emph{gaps} of $d$-row $I$ and $d$-column $J$,
\begin{align*}
  \Delta^\rowvar_I = \det^2(U(I^\ast, :)) - \det^2(U(I, :))\,, \quad
  \Delta^\colvar_J = \det^2(V(J^\ast, :)) - \det^2(V(J, :))\,,
\end{align*}
measure the hardness of eliminating $I$ and $J$ under the assumption that $U$ and $V$ are known. We define the \emph{minimum gap} as the minimum of the $d$-row and $d$-column gaps,
\begin{align}
  \Delta_{\min} =
  \min \set{\min_{I \in \Pi_d([K]) \setminus \set{I^\ast}} \Delta^\rowvar_I, \
  \min_{J \in \Pi_d([L]) \setminus \set{J^\ast}} \Delta^\rowvar_J}\,.
  \label{eq:minimum gap}
\end{align}
However, $U$ and $V$ are not known, and therefore $\lowrankelim$ estimates scaled volumes of $d$-rows and $d$-columns. The penalty for estimating scaled volumes is reflected by the \emph{minimum volume} $c_{\min}$ in \eqref{eq:minimum volume} and the \emph{maximum volume}
\begin{align}
  c_{\max} = \min \set{\det^2(U^\ast), \ \det^2(V^\ast)}\,.
  \label{eq:maximum volume}
\end{align}
Note that $c_{\min} > 0$ implies that any $d$ rows and columns of $\bar{R}$ are linearly independent. We discuss how to eliminate the dependence on $c_{\min}$ in \cref{sec:discussion}. Our main theorem is stated below.

\begin{theorem}
\label{thm:upper bound} The expected $n$-step regret of $\lowrankelim$ is bounded as
\begin{align*}
  \mathcal{R}(n) \leq \frac{c d^3 (K + L)}{c_{\max} c_{\min}^2 \Delta_{\min}} C(n) + 4\,,
\end{align*}
where $c = 3072$, $c_{\max}$ is defined in \eqref{eq:maximum volume}, $c_{\min}$ is defined in \eqref{eq:minimum volume}, $\Delta_{\min}$ is defined in \eqref{eq:minimum gap}, and $C(n)$ is defined in \eqref{eq:log n}.
\end{theorem}
\begin{proof}
Let $m$ be the first stage such that all suboptimal $d$-rows and $d$-columns are eliminated by its end, $\tilde{\Delta}_m < c_{\min} \Delta_{\min} / 2$. Let $\mathcal{R}_\ell$ be the expected regret of $\lowrankelim$ in stage $\ell$ under event $\cE$. Then the expected $n$-step regret of $\lowrankelim$ is bounded as
\begin{align*}
  \mathcal{R}(n) \leq
  \E{\mathcal{R}(n) \I{\cE}} + n P(\ccE) \leq
  \E{\mathcal{R}(n) \I{\cE}} + 4 \leq
  \sum_{\ell = 0}^m \mathcal{R}_\ell + 4\,,
\end{align*}
where the second inequality is from $P(\ccE) \leq 4 n^{-1}$ (\cref{lem:concentration}) and the last inequality holds because all suboptimal $d$-rows and $d$-columns are eliminated after stage $m$ (\cref{lem:elimination}).

When $d$-row $I$ or $d$-column $J$ is active in stage $\ell$, it has not been eliminated in the previous stages. Therefore, by \cref{lem:elimination}, $c_{\min} \Delta^\rowvar_I \leq 2 \tilde{\Delta}_{\ell - 1} = 4 \tilde{\Delta}_\ell$ and $c_{\min} \Delta^\colvar_J \leq 2 \tilde{\Delta}_{\ell - 1} = 4 \tilde{\Delta}_\ell$. Furthermore, by the design of exploration in $\lowrankelim$ (lines $4$--$12$), each remaining row and column is explored in some remaining $I$ and $J$, respectively, that contains it. Therefore, by the regret decomposition in \cref{lem:regret decomposition}, the $n$-step regret is bounded from above as
\begin{align*}
  \sum_{\ell = 0}^m \mathcal{R}_\ell \leq
  2 \sum_{\ell = 0}^m
  \frac{6 d^3 (4 \tilde{\Delta}_\ell + 4 \tilde{\Delta}_\ell)}{c_{\max} c_{\min}} (K + L) n_\ell =
  \frac{96 d^3 (K + L)}{c_{\max} c_{\min}} \sum_{\ell = 0}^m \tilde{\Delta}_\ell n_\ell\,.
\end{align*}
The additional factor of $2$ is because $\lowrankelim$ explores everything twice. Note that the above upper bound is only possible because $d$-rows and $d$-columns are eliminated simultaneously.

Now we express $n_\ell$ and note that $\tilde{\Delta}_m = \tilde{\Delta}_{m - 1} / 2 \geq c_{\min} \Delta_{\min} / 4$ from the definition of $m$,
\begin{align*}
  \sum_{\ell = 0}^m \tilde{\Delta}_\ell n_\ell \leq
  4 C(n) \sum_{\ell = 0}^m \frac{1}{\tilde{\Delta}_\ell} =
  \frac{4}{\tilde{\Delta}_m} C(n) \sum_{\ell = 0}^m 2^{- \ell} \leq
  \frac{32}{c_{\min} \Delta_{\min}} C(n)\,.
\end{align*}
Finally, we chain all above inequalities and get our main claim.
\end{proof}

\subsection{Key Lemmas}
\label{sec:key lemmas}

We state our key lemmas below, together with sketches of their proofs.

\begin{lemma}
\label{lem:concentration} Let
\begin{align*}
  \bar{\rnd{\mu}}^\rowvar_\ell(I) =
  \abs{\rnd{A}^\colvar_\ell}^{-1} \sum_{J \in \rnd{A}^\colvar_\ell} \det^2(\bar{R}(I, J))\,, \quad
  \bar{\rnd{\mu}}^\colvar_\ell(J) =
  \abs{\rnd{A}^\rowvar_\ell}^{-1} \sum_{I \in \rnd{A}^\rowvar_\ell} \det^2(\bar{R}(I, J))
\end{align*}
be the expected scaled volumes of $d$-row $I$ and $d$-column $J$ in stage $\ell$, and let
\begin{align*}
  \cE^\rowvar_{\ell, I} =
  \set{\bar{\rnd{\mu}}^\rowvar_\ell(I) \in [\rnd{L}^\rowvar_\ell(I), \rnd{U}^\rowvar_\ell(I)]}\,, \quad
  \cE^\colvar_{\ell, J} =
  \set{\bar{\rnd{\mu}}^\colvar_\ell(J) \in [\rnd{L}^\colvar_\ell(J), \rnd{U}^\colvar_\ell(J)]}
\end{align*}
be the events that the confidence intervals on these expected volumes hold. Let $\cE$ be the event that all confidence intervals hold and $\ccE$ be the complement of this event. Then $P(\ccE) \leq 4 n^{-1}$.
\end{lemma}
\begin{proof}
First, we prove that $\E{\hat{\rnd{\mu}}^\rowvar_\ell(I)} = \bar{\rnd{\mu}}^\rowvar_\ell(I)$ and $\hat{\rnd{\mu}}^\rowvar_\ell(I) \in [0, \det_{\max}^2(d)]$ for any stage $\ell$, $d$-row $I \in \Pi_d([K])$, and remaining $d$-columns $\rnd{A}^\colvar_\ell$ in stage $\ell$. Therefore, we can argue that $\hat{\rnd{\mu}}^\rowvar_\ell(I)$ is close to $\bar{\rnd{\mu}}^\rowvar_\ell(I)$ by Hoeffding's inequality. The column argument is analogous. Finally, by the union bound, we argue that it is unlikely that $\hat{\rnd{\mu}}^\rowvar_\ell(I)$ and $\hat{\rnd{\mu}}^\colvar_\ell(J)$ are not close to $\bar{\rnd{\mu}}^\rowvar_\ell(I)$ and $\bar{\rnd{\mu}}^\colvar_\ell(J)$, respectively, in any stage $\ell$. The complete proof is in \cref{sec:concentration lemma}.
\end{proof}

\begin{lemma}
\label{lem:regret decomposition} Let $I$ and $J$ be any $d$-row and $d$-column, respectively. Then
\begin{align*}
  \bar{R}(i^\ast, j^\ast) - \bar{R}(i^\ast(I, J), j^\ast(I, J)) \leq
  6 d^3 \frac{\Delta^\rowvar_I + \Delta^\colvar_J}{c_{\max}}\,.
\end{align*}
\end{lemma}
\begin{proof}
First, we bound the regret from above by the differences in its row and column components, $U^\ast - U(I, :)$ and $V^\ast - V(J, :)$. Then we argue that $U^\ast - U(I, :)$ can be bounded as a function of $\det^2(U^\ast) - \det^2(U(I, :))$, which is proved in \cref{lem:simplex matching} in \cref{sec:technical lemmas}. The column argument is analogous. The complete proof is in \cref{sec:regret decomposition lemma}.
\end{proof}

\begin{lemma}
\label{lem:elimination} Let event $\cE$ happen and $m$ be the first stage where $\tilde{\Delta}_m < c_{\min} \Delta^\rowvar_I / 2$, where $c_{\min}$ is defined in \eqref{eq:minimum volume}. Then $d$-row $I$ is guaranteed to be eliminated by the end of stage $m$. Moreover, let $m$ be the first stage where $\tilde{\Delta}_m < c_{\min} \Delta^\colvar_J / 2$. Then $d$-column $J$ is guaranteed to be eliminated by the end of stage $m$.
\end{lemma}
\begin{proof}
From the definition of our confidence intervals in $\lowrankelim$, $\rnd{U}^\rowvar_\ell(I) < \rnd{L}^\rowvar_\ell(I^\ast)$ happens when $2 \tilde{\Delta}_m \leq \bar{\rnd{\mu}}^\rowvar_\ell(I^\ast) - \bar{\rnd{\mu}}^\rowvar_\ell(I)$. Now note that $\bar{\rnd{\mu}}^\rowvar_\ell(I^\ast) - \bar{\rnd{\mu}}^\rowvar_\ell(I)$ is bounded from below by $c_{\min} \Delta^\rowvar_I$. The column argument is analogous. The complete proof is in \cref{sec:elimination lemma}.
\end{proof}


\subsection{Discussion}
\label{sec:discussion}

We derive a gap-dependent upper bound on the $n$-step regret of $\lowrankelim$ in \cref{thm:upper bound}. The bound does not depend on $K L$; is linear in the reciprocal of the minimum gap $\Delta_{\min}$ in \eqref{eq:minimum gap} and logarithmic in $n$ through $C(n)$ in \eqref{eq:log n}. To the best of our knowledge, $\lowrankelim$ is the first algorithm that achieves such regret. The polynomial dependence on rank $d$ is suboptimal and we believe that it can be reduced by a more elaborate analysis. The goal of our work is not to conduct such an analysis, but to demonstrate that these kinds of bounds are attainable by bandit algorithms.

Our regret bound also depends on the reciprocal of two problem-dependent quantities, $c_{\max}$ in \eqref{eq:maximum volume} and $c_{\min}$ in \eqref{eq:minimum volume}, which do not depend on $K$, $L$, $\Delta_{\min}$, and $n$. The maximum volume $c_{\max}$ arises in \cref{lem:simplex matching} in \cref{sec:technical lemmas}, which relates volume to regret. This quantity is not critical because it is unlikely to be small. In fact, $c_{\max} \geq c_{\min}$ by definition. The minimum volume $c_{\min}$ is the penalty for estimating scaled volumes of $d$-rows and $d$-columns, over random $d$-columns and $d$-rows, respectively. This is a form of averaging. Therefore, $c_{\min}$ is expected to be proportional to some notion of an average determinant, and not the minimum determinant as in \eqref{eq:minimum volume}.

In the rest of this section, we suggest a modification of $\lowrankelim$ whose regret scales much better with the minimum volume. The key idea is to follow Katariya \etal~\cite{katariya17stochastic} and slightly change the exploration step. The change is to choose the $d$-rows and $d$-columns in line $5$ of $\lowrankelim$ randomly from $\Pi_d([K])$ and $\Pi_d([L])$, respectively. If the chosen $d$-row is eliminated in an earlier stage, it is replaced with the $d$-row that eliminated it; or the $d$-row that eliminated the earlier eliminating $d$-row, and so on. The same strategy is applied to $d$-columns. The result is that the averaging penalty does not worsen with elimination. Then $c_{\min}$ in \cref{lem:elimination} can be substituted with $\bar{c} = \exp[\min \set{\bar{c}_\rowvar, \bar{c}_\colvar}]$, where
\begin{align*}
  \bar{c}_\rowvar =
  \abs{\Pi_d([K])}^{-1} \hspace{-0.1in} \sum_{I \in \Pi_d([K])} \hspace{-0.1in} \det^2(U(I, :))\,, \quad
  \bar{c}_\colvar =
  \abs{\Pi_d([L])}^{-1} \hspace{-0.1in} \sum_{J \in \Pi_d([L])} \hspace{-0.1in} \det^2(V(J, :))\,;
\end{align*}
and the expected $n$-step regret of $\lowrankelim$ becomes
\begin{align*}
  \mathcal{R}(n) \leq \frac{c d^3 (K + L)}{c_{\max} \bar{c}^2 \Delta_{\min}} C(n) + 4\,.
\end{align*}


\section{Related Work}
\label{sec:related work}

The closest related paper to our work are stochastic rank-$1$ bandits of Katariya \etal~\cite{katariya17stochastic}. This work can be viewed as a generalization of rank-$1$ bandits to a higher rank. Although our algorithm and analysis are motivated by Katariya \etal~\cite{katariya17stochastic}, our generalization is highly non-trivial. For instance, it is easy to see that the maximum entry of a non-negative rank-$1$ lies in its row and column with highest latent values. This is not true when the rank $d > 1$. Therefore, it may seem that the work of Katariya \etal~\cite{katariya17stochastic} cannot be generalized to a higher rank; and even if, it is unclear under what assumptions. We not only generalize this work, but also recover similar regret dependence.

Several papers studied various forms of low-rank matrix completion in the bandit setting. Zhao \etal~\cite{zhao13interactive} proposed a bandit algorithm for low-rank matrix completion, where the distribution over latent item factors is approximated by a point estimate. The algorithm is not analyzed. Kawale \etal~\cite{kawale15efficient} proposed a Thompson sampling algorithm for low-rank matrix completion, where the distribution over low-rank matrices is approximated by particle filtering. A computationally-inefficient variant of the algorithm has $O(\Delta^{-2} \log n)$ regret in rank-$1$ matrices. Sen \etal~\cite{sen17contextual} proposed an $\eps$-greedy algorithm for non-negative matrix completion. Its regret is $O(\Delta^{-2} \log n)$ and its analysis relies on a variant of the restricted isometry property, which may be hard to satisfy in practice. The following three papers studied clustering in the bandit setting, which is a form of a low-rank structure. Gentile \etal~\cite{gentile14online} clustered users based on their preferences, under the assumption that the features of items are known. Li \etal~\cite{li16collaborative} generalized this algorithm to the clustering of items. Maillard \etal~\cite{maillard14latent} studied a multi-armed bandit problem where the arms are partitioned into latent groups. All above papers also differ from our work in the setting. Our learning agent chooses both the row and column. In all above papers, the nature chooses the row.

Bhargava \etal~\cite{bhargava17active} studied active matrix completion of positive semi-definite matrices and discussed its applications to bandits. This work is not comparable to our paper because the classes of completed matrices are different.

Matrix recovery and completion have been studied extensively in both machine learning and statistics \cite{candes09exact,koren09matrix,candes10matrix,krishnamurth13lowrank}. A good recent review of the prior work is Davenport and Romberg \cite{davenport16overview}. The existing guarantees in noisy matrix completion are unsuitable for our setting because they are on $\|\bar{R} - \hat{R}\|_{F}$, where $\bar{R} \in [0, 1]^{K \times L}$ is the unobserved matrix and $\hat{R} \in [0, 1]^{L \times K}$ is its recovered approximation. For the sake of concreteness, suppose that the noise is $\mathcal{N}(0, 1)$. Then, by Theorem 7 in Candes and Plan \cite{candes10matrix}, $\|\bar{R} - \hat{R}\|_{F} \leq \sqrt{\max \set{K, L}}$ at best. This bound is not sufficient for our purpose, because the gap between the highest and second highest entries of $\bar{R}$ is by definition smaller than $1$. In fact, many entries of the matrix may need to be observed many times to learn its maximum entry, as this may not be possible from observing only a small portion of the matrix.


\section{Conclusions}
\label{sec:conclusions}

We propose an algorithm for finding the maximum entry of a class of stochastic low-rank matrices, which we call $\lowrankelim$. $\lowrankelim$ is computationally and sample efficient when the rank of the matrix is small. We derive a gap-dependent upper bound on the $n$-step regret of our algorithm that does not depend on $K L$, the product of the number of rows $K$ and columns $L$ in the matrix. The bound is linear in the reciprocal of the minimum gap and logarithmic in the number of steps $n$. Although such bounds have become common in many bandit problems, we are unaware of any such bound in stochastic low-rank matrix completion. To the best of our knowledge, this paper presents the first such result. Note that our bound is proved without making any incoherence assumption on matrices, as is common in matrix completion \cite{davenport16overview}. This clearly indicates that the problem of learning the maximum entry of a matrix is fundamentally different from matrix completion.

We leave open several questions of interest. The strongest assumption in our work is that any row of $U$ and $V$ can be written as a convex combination of $d$ base rows and columns, respectively. We believe that this assumption can be relaxed. In particular, under the assumption that all entries of $U$ and $V$ are non-negative, the maximum entry of $\bar{R} = U V\transpose$ at the vertices of the convex hulls over the rows of $U$ and $V$, respectively. These convex hulls are maximum volume convex objects in row and column latent spaces, similarly to $U^\ast$ and $V^\ast$ in \cref{sec:introduction}. Therefore, we believe that they can be learned, at least in theory, by a similar algorithm to $\lowrankelim$.

Another limitation of our work is the dependence on rank $d$. The polynomial dependence on $d$ in our regret bound (\cref{thm:upper bound}) is likely to be suboptimal, and we believe that it can be reduced by a more elaborate analysis. In addition, $\lowrankelim$ is not computationally efficient when $d$ is large. We believe that it can be implemented computationally efficiently because our class of matrices can be factored using linear programming \cite{recht12factoring}. Finally, $\lowrankelim$ is not sample efficient when $d$ is large. We believe that our algorithm can be implemented sample efficiently if the distributions of the determinant products in $\lowrankelim$ are sub-Gaussian in $\poly(d)$. This may be possible because the expectations of the determinant products is in $[0, 1]$.

\bibliographystyle{plain}
\bibliography{References}


\clearpage
\onecolumn
\appendix

\section{Proof of \cref{lem:concentration}}
\label{sec:concentration lemma}

Fix any stage $\ell$, $d$-row $I \in \Pi_d([K])$, and remaining $d$-columns $\rnd{A}^\colvar_\ell$ in stage $\ell$. Then
\begin{align*}
  \det(\rnd{R}^\rowvar_{\ell, t, 1}(I, :)) \det(\rnd{R}^\rowvar_{\ell, t, 2}(I, :))
\end{align*}
is an i.i.d. random variable in $t$ with two properties. First, any of its realizations is bounded as
\begin{align*}
  \det(\rnd{R}^\rowvar_{\ell, t, 1}(I, :)) \det(\rnd{R}^\rowvar_{\ell, t, 2}(I, :)) \leq
  \det_{\max}^2(d)
\end{align*}
because both $\rnd{R}^\rowvar_{\ell, t, 1}(I, :)$ and $\rnd{R}^\rowvar_{\ell, t, 2}(I, :)$ are random $d \times d$ matrices on $[0, 1]$. Second,
\begin{align*}
  \E{\det(\rnd{R}^\rowvar_{\ell, t, 1}(I, :)) \det(\rnd{R}^\rowvar_{\ell, t, 2}(I, :))}
  & = \E{\condE{\det(\rnd{R}^\rowvar_{\ell, t, 1}(I, :)) \det(\rnd{R}^\rowvar_{\ell, t, 2}(I, :))}{\rnd{J}_t}} \\
  & = \E{\condE{\det(\rnd{R}^\rowvar_{\ell, t, 1}(I, :))}{\rnd{J}_t}
  \condE{\det(\rnd{R}^\rowvar_{\ell, t, 2}(I, :))}{\rnd{J}_t}} \\
  & = \E{\det^2(\bar{R}(I, \rnd{J}_t)} \\
  & = \abs{\rnd{A}^\colvar_\ell}^{-1} \sum_{J \in \rnd{A}^\colvar_\ell} \det^2(\bar{R}(I, J)) \\
  & = \bar{\rnd{\mu}}^\rowvar_\ell(I)\,,
\end{align*}
where the first equality is from the tower rule, the second equality is from two independent observations of $\bar{R}(I, \rnd{J}_t)$, the third equality is because $\E{\det(\rnd{Z})} = \det(\E{\rnd{Z}})$ for any $\rnd{Z}$ whose entries have independent noise, and the fourth equality is because $\rnd{J}_t$ is chosen uniformly at random from $\rnd{A}^\colvar_\ell$. Therefore, by Hoeffding's inequality and from the definition of $C(n)$ in \eqref{eq:log n},
\begin{align*}
  P(\bar{\rnd{\mu}}^\rowvar_\ell(I) \notin [\rnd{L}^\rowvar_\ell(I), \rnd{U}^\rowvar_\ell(I)])
  & = P\left(\abs{\hat{\rnd{\mu}}^\rowvar_\ell(I) - \bar{\rnd{\mu}}^\rowvar_\ell(I)} \geq
  \sqrt{C(n) n_\ell^{-1}}\right) \\
  & \leq 2 \exp[-2 C(n) (2 \det_{\max}(d))^{-2}] \\
  & \leq 2 \exp[-2 \log((K^d + L^d) n)] \\
  & \leq 2 K^{- d} n^{-2}\,.
\end{align*}
By the same line of reasoning,
\begin{align*}
  P(\bar{\rnd{\mu}}^\colvar_\ell(J) \notin [\rnd{L}^\colvar_\ell(J), \rnd{U}^\colvar_\ell(J)]) \leq
  2 L^{- d} n^{-2}
\end{align*}
for any stage $\ell$, $d$-column $J \in \Pi_d([L])$, and remaining $d$-rows $\rnd{A}^\rowvar_\ell$ in stage $\ell$. Finally, by the union bound and from the above inequalities,
\begin{align*}
  P(\ccE)
  & \leq \sum_{\ell = 0}^{n - 1} \sum_{I \in \Pi_d([K])}
  P(\bar{\rnd{\mu}}^\rowvar_\ell(I) \notin [\rnd{L}^\rowvar_\ell(I), \rnd{U}^\rowvar_\ell(I)]) +
  \sum_{\ell = 0}^{n - 1} \sum_{J \in \Pi_d([L])}
  P(\bar{\rnd{\mu}}^\colvar_\ell(J) \notin [\rnd{L}^\colvar_\ell(J), \rnd{U}^\colvar_\ell(J)]) \\
  & \leq 4 n^{-1}\,.
\end{align*}
This concludes our proof.

\section{Proof of \cref{lem:regret decomposition}}
\label{sec:regret decomposition lemma}

Let $U = U(I, :)$ and $V = V(J, :)$. Fix any permutations $\pi^\rowvar$ and $\pi^\colvar$ over $[d]$. Then the regret can be decomposed into its row and column components as
\begin{align}
  & \bar{R}(i^\ast, j^\ast) - \bar{R}(i^\ast(I, J), j^\ast(I, J))
  \nonumber \\
  & \quad = \max_{i, j \in [d]} U^\ast(i, :) V^\ast(j, :)\transpose - \max_{i, j \in [d]} U(i, :) V(j, :)\transpose
  \nonumber \\
  & \quad = \max_{i, j \in [d]} U^\ast(i, :) V^\ast(j, :)\transpose -
  \max_{i, j \in [d]} U(\pi^\rowvar(i), :) V^\ast(j, :)\transpose + {}
  \nonumber \\
  & \quad \hspace{0.165in} \max_{i, j \in [d]} U(i, :) V^\ast(j, :)\transpose -
  \max_{i, j \in [d]} U(i, :) V(\pi^\colvar(j), :)\transpose
  \nonumber \\
  & \quad \leq \sum_{i, j = 1}^d \abs{(U^\ast(i, :) - U(\pi^\rowvar(i), :)) V^\ast(j, :)\transpose} +
  \sum_{i, j = 1}^d \abs{(V^\ast(i, :) - V(\pi^\colvar(i), :)) U(j, :)\transpose}\,.
  \label{eq:regret decomposition}
\end{align}
Now we focus on the first term. By definition, $U = Z U^\ast$ for some $Z \in \mathcal{S}_{d, d}$, and therefore
\begin{align*}
  \sum_{i, j = 1}^d \abs{(U^\ast(i, :) - U(\pi^\rowvar(i), :)) V^\ast(j, :)\transpose} \leq
  \sum_{i = 1}^d \abs{E(i, :) - Z(\pi^\rowvar(i), :))} \sum_{j = 1}^d \abs{U^\ast V^\ast(j, :)\transpose}\,,
\end{align*}
where $E = I_d$. By the Cauchy-Schwarz inequality,
\begin{align*}
  \sum_{i = 1}^d \abs{E(i, :) - Z(\pi^\rowvar(i), :))} \sum_{j = 1}^d \abs{U^\ast V^\ast(j, :)\transpose}
  & \leq \sum_{i = 1}^d \normw{E(i, :) - Z(\pi^\rowvar(i), :)}{2}
  \Bigg\|\sum_{j = 1}^d \abs{U^\ast V^\ast(j, :)\transpose}\Bigg\|_2 \\
  & \leq d^\frac{3}{2} \sum_{i = 1}^d \normw{E(i, :) - Z(\pi^\rowvar(i), :)}{2}\,.
\end{align*}
Let $\pi^\rowvar$ be the permutation in \cref{lem:simplex matching}. Then we apply the lemma and get that
\begin{align*}
  d^\frac{3}{2} \sum_{i = 1}^d \normw{E(i, :) - Z(i, :)}{2} \leq
  6 d^3 (1 - \det^2(Z)) =
  6 d^3 \frac{\det^2(U^\ast) - \det^2(U)}{\det^2(U^\ast)}\,.
\end{align*}
The second term in \eqref{eq:regret decomposition} can be bounded analogously as
\begin{align*}
  \sum_{i, j = 1}^d \abs{(V^\ast(i, :) - V(\pi^\colvar(i), :)) U(j, :)\transpose}
  \leq 6 d^3 \frac{\det^2(V^\ast) - \det^2(V)}{\det^2(V^\ast)}\,.
\end{align*}
Now we put both upper bounds together and get that
\begin{align*}
  \bar{R}(i^\ast, j^\ast) - \bar{R}(i^\ast(I, J), j^\ast(I, J)) \leq
  6 d^3 \frac{\Delta^\rowvar_I + \Delta^\colvar_J}{\min \set{\det^2(U^\ast), \ \det^2(V^\ast)}}\,.
\end{align*}
This concludes our proof.

\section{Proof of \cref{lem:elimination}}
\label{sec:elimination lemma}

We only prove the first claim. The other claim can be proved analogously.

Let $\bar{\rnd{\mu}}^\rowvar_\ell(I)$ and $\hat{\rnd{\mu}}^\rowvar_\ell(I)$ be defined as in \cref{lem:concentration}. Let $\bar{\rnd{\mu}}^\rowvar_\ell(I^\ast)$ and $\hat{\rnd{\mu}}^\rowvar_\ell(I^\ast)$ be the corresponding values for $I^\ast$. Then from the definition of our confidence intervals and that event $\cE$ happens,
\begin{align*}
  \rnd{U}^\rowvar_\ell(I)
  & \leq \hat{\rnd{\mu}}^\rowvar_\ell(I) + \frac{\tilde{\Delta}_m}{2} \leq
  \bar{\rnd{\mu}}^\rowvar_\ell(I) + \tilde{\Delta}_m\,, \quad \\
  \rnd{L}^\rowvar_\ell(I^\ast)
  & \geq \hat{\rnd{\mu}}^\rowvar_\ell(I^\ast) - \frac{\tilde{\Delta}_m}{2} \geq
  \bar{\rnd{\mu}}^\rowvar_\ell(I^\ast) - \tilde{\Delta}_m\,.
\end{align*}
To complete the proof, it remains to show that
\begin{align*}
  \bar{\rnd{\mu}}^\rowvar_\ell(I^\ast) - \tilde{\Delta}_m -
  (\bar{\rnd{\mu}}^\rowvar_\ell(I) + \tilde{\Delta}_m) \geq
  0\,.
\end{align*}
This follows from $\bar{\rnd{\mu}}^\rowvar_\ell(I^\ast) - \bar{\rnd{\mu}}^\rowvar_\ell(I) \geq c_{\min} \Delta^\rowvar_I$ and our assumption that $\tilde{\Delta}_m < c_{\min} \Delta^\rowvar_I / 2$.

\section{Technical Lemmas}
\label{sec:technical lemmas}

\begin{lemma}
\label{lem:simplex matching} Let $Z \in \mathcal{S}_{d, d}$ and $E = I_d$. Then there exists a permutation $\pi$ over $[d]$ such that
\begin{align*}
  \sum_{i = 1}^d \normw{E(i, :) - Z(\pi(i), :)}{2} \leq 6 d^\frac{3}{2} (1 - \det^2(Z))\,.
\end{align*}
\end{lemma}
\begin{proof}
Our proof has two parts. First, suppose that $\abs{\det(Z)} \leq 5 / 6$. Then our claim holds trivially because the distance of any two points in the $d$-dimensional simplex is bounded by
\begin{align*}
  \sqrt{d} \leq
  6 \sqrt{d} (1 - \abs{\det(Z)}) \leq
  6 \sqrt{d} (1 - \det^2(Z))\,.
\end{align*}
Now suppose that $\abs{\det(Z)} \geq 5 / 6$. Then any row of $Z$ must contain an entry whose value is at least $\det^2(Z)$. We prove this claim by contradiction. Suppose that $\maxnorm{Z(i, :)} < \det^2(Z)$ for some row $i$. Then
\begin{align*}
  \det^2(Z) \leq
  \prod_{k = 1}^d \normw{Z(k, :)}{2}^2 \leq
  \normw{Z(i, :)}{2}^2 \leq
  \normw{Z(i, :)}{1} \maxnorm{Z(i, :)} <
  \det^2(Z)\,,
\end{align*}
where the first inequality is Hadamard's determinant inequality, the second inequality follows from the observation that $\normw{Z(k, :)}{2} \leq 1$ for all $k \in [d]$, the third inequality is H{\"o}lder's inequality, and the last inequality follows from $\normw{Z(i, :)}{1} \leq 1$ and $\maxnorm{Z(i, :)} < \det^2(Z)$. The above inequality is clearly false, and therefore it must be true that $\maxnorm{Z(i, :)} \geq \det^2(Z)$ for all $i \in [d]$.

Note that any row of $Z$ has only one entry whose value is at least $\det^2(Z)$, because $\det^2(Z) > 1 / 2$. These entries are at distinct columns. We prove this by contradiction. Without loss of generality, let $Z(1, 1) \geq \det^2(Z)$ and $Z(2, 1) > \det^2(Z)$. Then from the Laplace expansion of the first row of $Z$, we have that
\begin{align*}
  \abs{\det(Z)} \leq
  Z(1, 1) \abs{\det(M_{1, 1})} + \sum_{j = 2}^d Z(1, j) \abs{\det(M_{1, j})}\,,
\end{align*}
where $M_{i, j}$ is a $(d - 1) \times (d - 1)$ matrix obtained from matrix $Z$ by removing its $i$-th row and $j$-th column. From $Z(1, 1) > \det^2(Z)$ and $1 - \abs{\det(Z)} \leq \abs{\det(Z)} / 5$, we have that
\begin{align*}
  \sum_{j = 2}^d Z(1, j) \leq
  1 - \det^2(Z) =
  (1 + \abs{\det(Z)}) (1 - \abs{\det(Z)}) \leq
  \frac{2}{5} \abs{\det(Z)}\,.
\end{align*}
Similarly, from $Z(2, 1) > \det^2(Z)$ and $1 - \abs{\det(Z)} \leq \abs{\det(Z)} / 5$, we have that
\begin{align*}
  \abs{\det(M_{1, 1})} \leq
  1 - \det^2(Z) =
  (1 + \abs{\det(Z)}) (1 - \abs{\det(Z)}) \leq
  \frac{2}{5} \abs{\det(Z)}\,.
\end{align*}
Now chain the above three inequalities, and note that $Z(1, 1) \leq 1$ and $\abs{\det(M_{1, j})} \leq 1$ for $j \geq 2$. The result is a contradiction that $\abs{\det(Z)} \leq (4 / 5) \abs{\det(Z)}$, and therefore it must be true that the maximum entries in each row of $Z$ are at distinct columns.

Based on the above, there exists a permutation $\pi$ over $d$ such that
\begin{align*}
  \abs{E(i, i) - Z(\pi(i), i)} \leq 1 - \det^2(Z)
\end{align*}
for any $i \in [d]$. Moreover, because $Z(\pi(i), j) \leq 1 - \det^2(Z)$ for any $j \neq \pi(i)$, we have that
\begin{align*}
  \abs{E(i, j) - Z(\pi(i), j)} \leq 1 - \det^2(Z)\,.
\end{align*}
It follows that
\begin{align*}
  \sum_{i = 1}^d \normw{E(i, :) - Z(\pi(i), :)}{2} \leq
  d^\frac{1}{2} \sum_{i = 1}^d \maxnorm{E(i, :) - Z(\pi(i), :)} \leq
  d^\frac{3}{2} (1 - \det^2(Z))\,.
\end{align*}
This concludes our proof.
\end{proof}

\end{document}